\documentclass{article}

     \usepackage[final]{neurips_2019}

\usepackage[utf8]{inputenc} 
\usepackage[T1]{fontenc}    
\usepackage{booktabs}       
\usepackage{amsfonts}       
\usepackage{nicefrac}       
\usepackage{microtype}      

\title{Towards Neural-Guided Program Synthesis\\for Linear Temporal Logic Specifications}

\author{
Alberto Camacho \\
University of Toronto \\
Vector Institute \\
\texttt{acamacho@cs.toronto.edu} \\
\And
Sheila A. McIlraith \\
University of Toronto \\
Vector Institute \\
\texttt{sheila@cs.toronto.edu} \\
}

\usepackage{microtype}
\usepackage{graphicx}
\usepackage{subcaption}
\usepackage{booktabs} 
\usepackage{hyperref}
\usepackage{amsmath,mathtools,amssymb}
\usepackage{bm}
\usepackage{wrapfig}

\usepackage{xspace}

\def\trace{\ensuremath{\pi}} 

\def\trace{\rho}

\newcommand{\st}{{\hspace{1mm}|\hspace{1mm}}}

\newcommand{\acite}[1]{\citeauthor{#1} \shortcite{#1}}
\renewcommand{\acite}[1]{\citep{#1}}
\renewcommand{\cite}[1]{\citep{#1}}

\newcommand{\set}[1]{\left\{ #1 \right\}}
\newcommand{\la}{\langle}
\newcommand{\ra}{\rangle}

\renewcommand{\st}{{\hspace{1mm}|\hspace{1mm}}}

\def\my-isdef{\hbox{~$\stackrel{\text{def}}{=}$~}}

\renewcommand{\paragraph}[1]{\vspace{1mm}\noindent\textbf{#1}}

\usepackage{amsthm}
\newtheorem{theorem}{Theorem}

\newtheorem{definition}{Definition}

\newcommand{\hide}[1]{}

\usepackage{colortbl}
\usepackage{xcolor}
\colorlet{tablerowcolor}{gray!10} 
\newcolumntype{L}[1]{>{\raggedright\arraybackslash}p{#1}}
\newcolumntype{C}[1]{>{\centering\arraybackslash}p{#1}}
\newcolumntype{R}[1]{>{@{\,}\raggedleft\arraybackslash}p{#1}@{\,}}

\def\LTL{\ensuremath{\mathsf{LTL}}\xspace}

\newcommand{\ltluntil}[2]{{#1} \mathcal{U} {#2}}

\newcommand{\ltlrelease}[2]{{#1} \mathcal{R} {#2}}
\newcommand{\ltlalways}[1]{\square{#1}}
\newcommand{\tikzcircle}[1]{\tikz[baseline=-0.5ex]\draw[black,radius=3pt,fill=#1] (0,0) circle ;}

\newcommand{\ltlnext}[1]{\tikzcircle{white}\xspace{#1}}
\newcommand{\ltleventually}[1]{\lozenge\xspace{#1}}

\usepackage[textsize=scriptsize,textwidth=1cm,disable]{todonotes} 

\usepackage{bbm}
\newcommand{\indicator}[1]{\mathbbm{1}({#1})}

\newcommand{\C}[1]{\mathcal{#1}\xspace}

\newcommand{\UCW}{\text{UCW}\xspace}

\def\defeq{\coloneqq}

\def\play{\rho}

\newcolumntype{b}{X}
\newcolumntype{s}{>{\hsize=.5\hsize}X}

\DeclareMathOperator{\idx}{idx}

\newcommand{\bigO}[0]{\mathcal{O}}

\DeclareMathOperator{\round}{round}
\DeclareMathOperator{\floor}{floor}

\begin{document}

\maketitle


\begin{abstract}
Synthesizing a program that realizes a logical specification is a classical problem in computer science.
%
%
%
%
%
 We examine a particular type of program synthesis, where the  objective is to synthesize a strategy 
that reacts to a potentially adversarial environment while ensuring that all executions satisfy a \emph{Linear Temporal Logic} (\LTL) specification. 
%
Unfortunately, exact methods to solve 
so-called \emph{\LTL synthesis} via logical inference do not scale.
In this work, we cast \LTL synthesis as an optimization problem.
We 
employ
a neural network to learn a Q-function that is then used to guide search, and to construct programs that are subsequently verified for correctness.
Our method is unique in combining search with deep learning to realize \LTL synthesis.
%
%
In our experiments 
the learned Q-function
provides effective guidance for synthesis problems with relatively small specifications.
%
%

\end{abstract}

\section{Introduction}

Automated synthesis of programs from logical specification---\emph{program synthesis from specification}---is a classical problem in computer science that dates back to Church~(\citeyear{Church57}).
We are concerned with \emph{\LTL synthesis}
---the synthesis of a reactive module that interacts with the environment such that all executions satisfy a prescribed \emph{Linear Temporal Logic} (\LTL) formula. The formula  specifies the objective of the program and other sundry constraints including assumptions on the behavior of the environment, and safety and liveness constraints~\cite{pnu-ros-popl89}.
There is no simulator, no requirement of a Markovian transition system, and no reward function.  Rather, the logical specification provides all that is known of the behavior of the environment and mandated behavior of the program.
%
Synthesized programs are \emph{correct by construction}.
%
%
%
%
\LTL synthesis has a myriad of applications including the automated construction of logical circuits, game agents, and controllers for intelligent devices and safety-critical systems.


%


Programming from specification, and specifically \LTL synthesis has been studied intensely in theory and in practice (e.g., \cite{jacobs-syntcomp19}). 
We contrast it with \emph{programming by example} which synthesizes a program that captures the behavior of input-output data or program traces (e.g., \cite{gulwani16}).
In this category, differential approaches use deep learning for \emph{program induction}, 
to \emph{infer} the output of the program for new inputs (e.g. \cite{EllisST16,ParisottoMS0ZK17,ZhangRFLBMUZ18,ChenLS18}).
%
In neural-guided search, statistical techniques are used to guide search (e.g. \cite{BalogGBNT17}),
motivated in part by the fact that discrete search techniques may be more efficient than differentiable approaches alone \cite{GauntBSKKTT16}.

\LTL synthesis is a challenging problem. Its complexity is known to be 2EXP-complete, and exact methods to 
solve this problem 
via logical inference do not scale---in part because of the prohibitively large search space.
%
Thus, there is a need for exploration and development of novel approaches to \LTL synthesis that have the potential to scale.

In this work we make a first step towards applying two scalable techniques to \LTL synthesis. Namely, \emph{search} and \emph{learning}.
We are interested in answering this question:
    \emph{Can we learn good guidance for \LTL synthesis?}

Sophisticated search techniques, 
and in particular heuristic search, 
have the potential to scale to large spaces 
by means of efficient exploration and pruning.
%
%
%
However, searching over spaces induced by LTL formulae present unique challenges that have not yet resulted in effective heuristic search techniques.
Recent successes of deep learning for Markov Decision Processes (MDP) (e.g. \cite{MnihKSRVBGRFOPB15_dqn}) and complex sequential decision-making problems (e.g. \cite{Silver2018AGR}) have inspired us to explore the use neural networks to capture the complex structure of \LTL specifications with a view to providing guidance to search methods.

The main challenge that we had to address in this work was how to frame \LTL synthesis---a problem that has commonly been addressed using formal methods via logical inference---into an optimization problem.
To this end, we introduce a novel dynamic programming based formulation of \LTL synthesis that allows us to deploy statistical approaches, and in particular deep learning combined with search, to realize \LTL synthesis.
%
We 
employ
a neural network to learn a Q-function that is used to guide search, and to construct programs that are subsequently verified for correctness, thereby benefitting from the scalability of an approximate method while maintaining correct-by-construction guarantees.
%
%
In our experiments the learned Q-function provided effective guidance for the synthesis of relatively small specifications, solving a number of the benchmarks that serve as standard evaluation metrics for the \LTL synthesis community.
\section{Reactive LTL Synthesis}

The central problem that we examine in this paper is 
the synthesis of controllers for sequential decision-making in discrete dynamical environments.
\emph{Reactive synthesis} 
constructs
a strategy such that \emph{all} executions of the strategy realize a specified temporally extended property, regardless of how the environment behaves \cite{pnu-ros-popl89}.
The specified property may include reachability goals and safety and liveness properties, as well as assumptions regarding the behavior of the environment.
In contrast with MDPs, the environment dynamics are not stochastic nor necessarily Markovian but rather non-deterministic.

In \LTL synthesis, the temporally extended property to be satisfied takes the form of a \emph{Linear Temporal Logic} (\LTL) formula over a set of \emph{environment} ($\C{X}$) and \emph{system} ($\C{Y}$) variables (Definition \ref{def:specification}).
\LTL is a modal logic that extends propositional logic with temporal modalities to express temporally extended properties of infinite-length state traces. In a nutshell, 
$\ltlnext{}\varphi$ 
denotes that $\varphi$ holds in the next timestep, 
and $\varphi_1 \ltluntil{}{} \varphi_2$ denotes that $\varphi_1$ holds until $\varphi_2$ holds. 
Unary operators \emph{eventually} ($\ltleventually{}$) and \emph{always} ($\ltlalways{}$), and binary operator \emph{release} ($\ltlrelease{}{}$) can be defined using $\ltlnext{}$ and $\ltluntil{}{}$  (cf. \cite{pnueli77}). 
Formally, 
we say that an infinite trace $\trace = s_1, s_2, \ldots$ satisfies $\varphi$ ($\trace\models\varphi$, for short) iff $\trace,1\models \varphi$, where for every natural number $i\geq 1$:


%
%
\begin{itemize}
\item $\trace,i\models p$, for a propositional variable $p$, iff $p \in s_i$,
\item $\trace,i\models \neg \psi$ iff it is not the case that $\trace,i\models\psi$,
\item $\trace,i\models (\psi\wedge\chi)$ iff $\trace,i\models\psi$ and $\trace,i\models \chi$,
\item $\trace,i\models \ltlnext{}\varphi$ iff $\trace,i+1\models\varphi$,
\item $\trace,i\models \varphi\ltluntil{}\psi$ iff there exists a $j\geq i$ such that $\trace,j\models\psi$, and $\trace,k\models\varphi$ for every $i \leq k < j$.
\end{itemize}


\begin{definition}\label{def:specification}
An \LTL specification is a triplet $\langle \C{X}, \C{Y}, \varphi \rangle$, where $\C{X}$ and $\C{Y}$ are two finite disjoint sets of environment and system variables, respectively, and $\varphi$ is an \LTL formula over $\C{X} \cup \C{Y}$.
\end{definition}

\LTL synthesis is the task of computing a strategy that satisfies the specification (Definition \ref{def:synthesis}).
Strategies are functions that map non-empty finite sequences of subsets of $\C{X}$---i.e., elements of $(2^{\C{X}})^+ $---to a subset of $\C{Y}$. In general, strategies are non-Markovian and take history into account. 

\begin{definition}\label{def:synthesis}
The \emph{synthesis} problem for an \LTL specification  $\langle \C{X}, \C{Y}, \varphi \rangle$ consists in computing an agent strategy $f: (2^{\C{X}})^+ \rightarrow 2^{\C{Y}}$ such that, for any infinite sequence $X_1 X_2 \cdots$ of subsets of $\C{X}$, the sequence $(X_1 \cup f(X_1)) (X_2 \cup f(X_1 X_2)) \cdots$ satisfies $\varphi$. The \emph{realiazability} problem consists in deciding whether one such strategy exists.
\end{definition}

\paragraph{Example:} The \LTL specification $\langle \set{x}, \set{y}, \ltlalways{(x \leftrightarrow \ltlnext{y})} \rangle$ is realizable. One agent strategy that synthesizes the specification
outputs $Y_{n+1} = \set{y}$ whenever $X_n = \set{x}$, and $Y_{n+1} = \emptyset$ whenever $X_n = \emptyset$.
Note, the \LTL formula does not constrain the output of the system in the first timestep.

\subsection{Bounded Synthesis}
Traditional approaches to \LTL realizability and synthesis rely on automata transformations of the \LTL formula. Here, we review results from so-called \emph{bounded synthesis} approaches, that make use of \emph{Universal Co-B\"uchi Word} (\UCW) automata \cite{KupfermanV05}.
An important result in bounded synthesis is that \LTL realizability can be characterized in terms of the runs of \UCW automata, in a way that the space of the search for solutions can be \emph{bounded}.
The result, adapted from \acite{ScheweF07a}, is stated in Theorem \ref{thm:bounded_synthesis}.
Besides that, the computational advantage of bounded synthesis is that \UCW automata transformations can be done more efficiently in practice than to other types of automata (e.g., parity automata).
%

\paragraph{\UCW automata:} 
An \UCW automaton is a tuple $\langle Q, \Sigma, q_0, \delta, Q_F \rangle$, where $Q$ is a finite set of states, $\Sigma$ is the input alphabet (here, $\Sigma \defeq 2^{\C{X} \cup \C{Y}}$), $q_0 \in Q$ is the initial state, $\delta: Q \times \Sigma \rightarrow 2^{Q}$ is the (non-deterministic) transition function, and $Q_F \subseteq Q$ is a set of rejecting states.
Without loss of generality, we assume $\delta(q,\sigma) \neq \emptyset$ for each $(q,\sigma) \in Q \times \Sigma$.
A \emph{run} of ${A}_\varphi$ on a play $\rho = (X_1\cup Y_1) (X_2\cup Y_2) \cdots$ is a sequence $q_0 q_1 \cdots$ of elements of $Q$, where each $q_{i+1}$ is an element of $\delta(q_i, X_i \cup Y_i)$.
The \emph{co-B\"uchi} index of a run is the maximum number of occurrences of rejecting states.
A run of ${A}_\varphi$ is \emph{accepting} if its co-B\"uchi index is finite,
and a play $\play$ is \emph{accepting} if all the runs of  ${A}_\varphi$ on $\play$ are accepting.
%
An \LTL formula $\varphi$ can be transformed into an \UCW automaton that accepts all and only the models of $\varphi$ in worst-case exponential time \cite{KupfermanV05}.
%


\begin{theorem}
\label{thm:bounded_synthesis}
Let $\C{A}_\varphi$ be a \UCW transformation of \LTL formula $\varphi$.
An \LTL specification  $\langle \C{X}, \C{Y}, \varphi \rangle$ is realizable iff there exists a strategy $f: (2^{\C{X}})^+ \rightarrow 2^{\C{Y}}$ and $k < \infty$, worst-case exponential in the size of $\C{A}_\varphi$, such that, for any infinite sequence $X_1 X_2 \cdots$ of subsets of $\C{X}$, all runs of $\C{A}_\varphi$ on the sequence $(X_1 \cup f(X_1)) (X_2 \cup f(X_1 X_2)) \cdots$ hit a number of rejecting states that is bounded by $k$.
\end{theorem}

\subsection{Automata Decompositions}
\label{sec:automata_decompositions_bounded_synthesis}

\UCW transformations of \LTL formulae are worst-case exponential, and can be a computational bottleneck. To mitigate for this, 
synthesis tools Acacia$^+$ \cite{BohyBFJR12-acaciaplus} and SynKit \cite{cam-mui-bai-mci-ijcai18} 
decompose the formula into a conjunction of subformulae, and then transform each subformula into \UCW automata.
Clearly, each subformula is potentially easier to transform into \UCW automata than the whole formula.
The results for bounded synthesis can be naturally extended to multiple automata decompositions, where the automata capture, collectively, \LTL satisfaction.

\begin{theorem}[adapted from \acite{BohyBFJR12-acaciaplus}]
\label{thm:bounded_synthesis_decomposed}
Let $\C{A}_i$ be \UCW transformations of \LTL formulae $\varphi_i$, $i=1..m$, and let $\varphi = \varphi_1 \land \cdots \land \varphi_m$.
\LTL specification  $\langle \C{X}, \C{Y}, \varphi \rangle$ is realizable iff there exists a strategy $f: (2^{\C{X}})^+ \rightarrow 2^{\C{Y}}$ and $k < \infty$ such that, for any infinite sequence $X_1 X_2 \cdots$ of subsets of $\C{X}$, the runs of each $\C{A}_i$ on the sequence $(X_1 \cup f(X_1)) (X_2 \cup f(X_1 X_2)) \cdots$ hit a number of rejecting states that is bounded by $k$.
\end{theorem}

\section{Safety Games for Bounded Synthesis}
\label{sec:bounded_synthesis_and_safety_games}


%
Reactive synthesis is usually interpreted as a two-player game between the \emph{system} player, and the \emph{environment} player.
In each turn, the environment player makes a move by selecting a subset of \emph{uncontrollable} variables, $X\subseteq \C{X}$.
In response, the system player makes a move by selecting a subset of \emph{controllable} variables, $Y \subseteq \C{Y}$---thus, the sets of player actions are the powersets of $\C{X}$ and $\C{Y}$.
%
%
Game states and transitions are constructed by means of \emph{automata} transformations of the \LTL formula. 
%
%
%
In particular, bounded synthesis is interpreted as a \emph{safety game} where the agent is constrained to react in a way that (infinite-length) game plays must yield automaton runs that hit a bounded number of rejecting states (cf. Theorems \ref{thm:bounded_synthesis} and \ref{thm:bounded_synthesis_decomposed}).
 We formalize safety games below, and frame bounded synthesis as safety games in Section \ref{sec:safety_game_reformulations}.



\paragraph{Safety games:}
In this paper, a
two-player \emph{safety game} is a tuple $\langle Z_\mathit{env}, Z_\mathit{sys}, S, s_1, T, S_{bad} \rangle$,
where $Z_\mathit{env}$ and $Z_\mathit{sys}$ are sets of actions, $S$ is a set of states, $s_1 \in S$ is the initial state of the game, $T: S\times Z_\mathit{env} \times Z_\mathit{sys} \rightarrow S$ is a transition function, and $S_{bad} \subseteq S$ is a set of losing states.
Analogously to \LTL synthesis, we refer to the players as the environment and system.
The game starts in $s_1$, and is played an infinite number of turns. In each turn, the environment player selects an action $x\in Z_\mathit{env}$, and the system player reacts by selecting an action $y\in Z_\mathit{sys}$. 
If the game state in the $n$-th turn is $s_{n}$ and the players moves are $x_n$ and $y_n$,
then the game state transitions to $s_{n+1} = T(s_n,x_{n},y_{n})$.
Thus, game plays are infinite sequences of pairs $(x,y) \in Z_\mathit{env} \times Z_\mathit{sys}$ that yield sequences of game states $\rho = s_1 s_2 \cdots$.
A game play is winning (for the system player) if it yields a sequence of states that never hits
a losing state.
%
Solutions to a safety game are \emph{policies} $\pi: S \times Z_\mathit{env} \rightarrow Z_\mathit{sys}$ that guarantee that game plays are winning, regardless how the environment moves, provided that the system acts in each state as mandated by $\pi$.
%
Safety games can be solved in polynomial time in the size of the search space,
via fix-point computation of the set of \emph{safe states}---i.e., states from which the system player has a winning strategy to avoid falling into losing states. If the initial state is a safe state, then a winning strategy for the system player can be obtained by performing actions that prevent the environment from transitioning into an unsafe state.





\subsection{Safety Game Reformulations} 
\label{sec:safety_game_reformulations}

Bounded synthesis approaches reduce the synthesis problem into a series of {safety} games $G_k$, parametrized by $k$.
In those games, the environment and system players perform actions that correspond to the powersets of $\C{X}$ and $\C{Y}$ variables, respectively.
%
%
%
First, $\varphi$ is transformed into a \UCW automaton, ${A}_\varphi = \langle Q, \Sigma, q_0, \delta, Q_F \rangle$.
With a fixed order in the elements of $\C{X}$, $\C{Y}$, and $Q$, we identify each subset $X \subseteq \C{X}$ with a boolean vector $\bm{x} = \la x^{(1)}, \ldots, x^{(|\C{X}|)} \ra$ that indicates whether $x_i \in X$ by making the $i$th element in $\bm{x}$ true (similarly with subsets of $\C{Y}$ and $Q$).
%
Game states are vectors $\bm{s} = \la s^{(1)}, \ldots, s^{(|Q|)} \ra$ that do bookkeeping of the co-B\"uchi indexes of the runs of ${A}_\varphi$ on the partial play that leads to $\bm{s}$.
More precisely, each element $s^{(i)}$ is an integer 
that tells the maximum co-B\"uchi index among all the runs of $A_\varphi$ on the partial play that finish in $q^{(i)}$. If none of the runs finish in $q^{(i)}$, then $s^{(i)} = -1$. 
The transition function, $T$, progresses the automaton runs captured in a state $\bm{s}$ according to the move of the environment and system players.
\begin{align*}
    {T}(\bm{s}, \bm{x}, \bm{y}) &\defeq \la s'^{(1)}, \ldots, s'^{(|Q|)} \ra \text{, where} \\
    t^{(i)} &\defeq  \max \{ s^{(j)} \st  q_i \in \delta(q_j, \bm{x} \Vert \bm{y}) \} \\
    s'^{(i)} &\defeq  t^{(i)} + \indicator{q_i \in Q_F}  \cdot \indicator{t^{(i)} > -1} \\
\end{align*}

We write $\idx(\bm{s})$ to refer to the \emph{co-B\"uchi index of game state} $\bm{s}$, that we define as the maximum co-B\"uchi index among all the runs captured by $\bm{s}$. Formally, 
$\idx(\bm{s}) \defeq \max_i s^{(i)}$.
Losing states in $S_{bad}$ are those $\bm{s}$ with  $\idx(\bm{s}) = k$.
%
%
%
%
%
%
Intuitively, in those games the environment player aims at maximizing the index of the states visited, whereas the system player aims at keeping this number bounded.
Theorem \ref{thm:bounded_synthesis_games} relates \LTL realizability with the existence of solutions to those safety games. 
For computational reasons, we redefine a transition function $T_k$ that reduces the search space to those states with co-B\"uchi index bounded by $k$. 

\begin{theorem}
[adapted from \cite{ScheweF07a}]\label{thm:bounded_synthesis_games}
\LTL specification $\langle \C{X}, \C{Y}, \varphi \rangle$ is realizable iff
the
safety game $ G_k = \langle 2^\C{X}, 2^\C{Y}, S, \bm{s_1}, T_k, S_{bad} \rangle$ 
constructed from any \UCW transformation of $\varphi$, $A_\varphi = \langle Q, \Sigma, q_0, \delta, Q_F \rangle$,
has a solution for some $k < \infty$, worst-case exponential in the number of states in ${A}_\varphi$, where:
\begin{align*}
    S &\defeq [-1, k]^Q \\
    \bm{s_1} &\defeq \langle 0, -1, \ldots, -1 \rangle \\
    {T_k}(\bm{s}, \bm{x}, \bm{y}) &\defeq \la s'^{(1)}, \ldots, s'^{(|Q|)} \ra \text{, where} \\
    s'^{(i)} &\defeq  \min (k,{T}(\bm{s}, \bm{x}, \bm{y})^{(i)} ) \\
    S_{bad} &\defeq \set{\bm{s} \in S \st \idx(\bm{s}) = k}
\end{align*}
\end{theorem}

\subsection{Automata Decompositions}

%

Theorem \ref{thm:bounded_synthesis_games} can be extended to handle multiple automata decompositions of the \LTL formula.
The result, stated in Theorem \ref{thm:bounded_synthesis_games_decomposed}, is adapted from known results that proof the correctness of bounded synthesis approaches employed by Acacia$^+$ and SynKit, which exploit decompositions to improve scalability \cite{BohyBFJR12-acaciaplus,cam-mui-bai-mci-ijcai18}.
We presume the \LTL specification formula is a conjunction of $m$ subformulae $\varphi = \varphi_1 \land \cdots \land \varphi_m$.
For each $i=1..m$, 
let $G_k^{(i)}$ be the safety game constructed as described in Theorem \ref{thm:bounded_synthesis_games} from a \UCW automaton transformation $\C{A}_i$ of $\varphi_i$.
%
The idea is to construct a cross-product safety game $G_k$
that maintains the dynamics of each $G_k^{(i)}$ in parallel.

\begin{theorem}\label{thm:bounded_synthesis_games_decomposed}
Let $\varphi = \varphi_1 \land \cdots \land \varphi_m$ be an \LTL formula, and
let $G_k^{(i)} = \langle 2^\C{X}, 2^\C{Y}, S^{(i)}, \bm{s_1^{(i)}}, T_k^{(i)}, S_{bad}^{(i)} \rangle$ be safety games associated to each $\varphi_i$, $i=1..m$, and constructed as described in Theorem \ref{thm:bounded_synthesis_games}.
\LTL specification $\langle \C{X}, \C{Y}, \varphi \rangle$ is realizable iff
the
safety game $ G_k = \langle 2^\C{X}, 2^\C{Y}, S, \bm{s_1}, {T_k}, S_{bad} \rangle$ 
constructed as described below
has a solution for some $k < \infty$.
\begin{align*}
    S &\defeq S^{(1)} \times \cdots \times S^{(m)}\\
    \bm{s_1} &\defeq \la \bm{s^{(1)}_1}, \ldots, \bm{s^{(m)}_1} \ra\\
    {T_k}(\bm{s}, \bm{x}, \bm{y}) &\defeq \la {T_k}^{(1)}(\bm{s^{(1)}}, \bm{x}, \bm{y}), \ldots, {T_k}^{(m)}(\bm{s^{(m)}}, \bm{x}, \bm{y}) \ra \\
    S_{bad} &\defeq \set{\bm{s} \st \idx(\bm{s}) = k}
\end{align*}
\end{theorem}





\subsection{Extraction of Solutions}
So far the results in Theorems \ref{thm:bounded_synthesis_games} and \ref{thm:bounded_synthesis_games_decomposed} relate \LTL realizability with the existence of solutions to the reduced safety games, but do not state how we can synthesize solutions to the original \LTL specification.
In effect,
a winning strategy $f: (2^{\C{X}})^+ \rightarrow 2^{\C{Y}}$ that realizes the \LTL specification  $\langle \C{X}, \C{Y}, \varphi \rangle$ can be directly constructed from a policy $\pi:  S \times X \rightarrow Y$ that solves the game $G_k$.
Recursively, 
\begin{align*}
    f(\bm{x}) &\defeq \pi(\bm{s_1}, \bm{x})\\
    f(\bm{x_1} \cdots \bm{x_m}) &\defeq \pi(\bm{s_m}, \bm{x_m}) \\
    \bm{s_{m+1}} &\defeq T(\bm{s_m},\bm{x_m}, \pi(\bm{s_m},\bm{x_m}))
\end{align*}
An advantageous property of solutions to the reduced safety games is that 
%
they
provide means to implement $f$ compactly in the form of a controller, or finite-state machine with internal memory that stores the current game state, and whose output to a sequence $\bm{x_1} \cdots \bm{x_m}$ depends only on the last environment move, $\bm{x_m}$, and the internal state $\bm{s_m}$---not on the entire history.

\subsection{Computational Complexity}
The transformation of \LTL into an \UCW automaton is worst-case exponential in the size of the formula.
For a fixed $k<\infty$, 
the size of the state space is $\bigO(k^{\Sigma |Q_i|})$, that is, exponential in the sum of automata sizes.
The game $G_k$ can be solved by fix-point computation in polynomial time in the size of the search space.
Decompositions of the \LTL formula have the potential to produce smaller \UCW automata and games with smaller size.
Note, however, that 
the worst-case computational complexity is doubly exponential in the size of the \LTL specification formula.

\section{Dynamic Programming for LTL Synthesis}
\label{sec:dyncamic_programming_for_ltl_synthesis}

In Section \ref{sec:bounded_synthesis_and_safety_games} we showed how \LTL synthesis can be cast as a series of safety games. Modern approaches to bounded synthesis use different technology (e.g. BDDs, SAT) to solve those games via logical inference (e.g. \cite{JobstmannB06_lily,BohyBFJR12-acaciaplus}). Unfortunately, exact methods have limited scalability, because the size of the search space grows (worst-case) doubly-exponentially with the size of the \LTL specification formula.
In this section we study the use of dynamic programming to solve these safety games.

Dynamic programming gradually approximates optimal solutions, and is an alternative to exact methods that do logical inference.
%
%
Dynamic programming has been widely used in the context of \emph{Markov Decision Processes} (MDPs), where the objective is to compute strategies that optimize for reward-worthy behavior.
The challenge for us in this work is to frame \LTL synthesis as an optimization problem, for which dynamic programming techniques can be used.
There are at least two significant differences between the dynamics of MDPs and safety games, that prevent us from using off-the-shelf methods. 
First, state transitions in an MDP are stochastic, whereas in a game the transitions are non-deterministic. Second, in an MDP the agent receives a reward signal upon performing an action, and optimizes for maximizing the expected (discounted) cumulative reward. In contrast, in a safety game the agent does not receive reward signals, and aims at winning the game.
The same differences appear with so-called \emph{Markov games} \cite{Littman94}.

As it is common in dynamic programming, we first formulate a set of Bellman equations for safety games. Bellman equations describe the \emph{value} of a state $\bm{s}$, $V(\bm{s}),$ in terms of the values of future states in a sequential decision-making process, assuming both players act ``optimally'' -- recall, the environment aims at maximizing the maximum index of visited states, and the system aims at keeping this number bounded. 
%
Solutions to the safety game can be obtained by performing, in each state, a greedy action that minimizes the (in our case, worst-case) value of the next state.
Bellman equations can be solved using dynamic programming -- for instance, value iteration.
We show that value iteration on this set of Bellman equations is guaranteed to converge to a solution to the associated safety game.
In the next section, we show how learning methods -- in particular, an adaptation of (deep) Q-learning -- can be used to provide guidance.
\subsection{Bellman Equations}

In our 
set of Bellman equations, we conceptualize $V(\bm{s})$ as a function that tells the maximum index of the states visited from $\bm{s}$, assuming that both players act optimally (cf. Theorem \ref{thm:bellman_one_ucw}).
%
Solutions are value
-- or, equivalently, Q-value -- functions that satisfy the Bellman equations in each $\bm{s} \in S$.
%


\begin{theorem}
\label{thm:bellman_one_ucw}
An \LTL specification $\langle \C{X}, \C{Y}, \varphi \rangle$ is realizable iff there exists a 
value
function $V$ with $V(\bm{s_1}) < \infty$ 
that is solution to the Bellman equations
\begin{align*}
    & V(\bm{s}) = \max_{\bm{x}} Q(\bm{s},\bm{x}) \\
    & Q(\bm{s},\bm{x}) = \max(\idx(\bm{s}), \min_{\bm{y}} V({T}(\bm{s}, \bm{x},\bm{y}))) 
\end{align*}
Furthermore, policy $\pi(\bm{s}, \bm{x}) \defeq \arg\min_{\bm{y}} V({T}(\bm{s}, \bm{x}, \bm{y}))$ solves the safety games $G_k$ for $k \geq V(\bm{s_1})$.
\end{theorem}
\begin{proof}
If  $\langle \C{X}, \C{Y}, \varphi \rangle$ is realizable, then the safety game $G_k$ has a solution $\pi$ for some $k<\infty$ (Theorem \ref{thm:bounded_synthesis_games}). $V(\bm{s})$ can be defined as the maximum index of the states reachable from $\bm{s}$ following $\pi$.
In the other direction, let $V$ be a solution to the Bellman equations. Then, the greedy policy $\pi(\bm{s},\bm{x}) = \arg\min_y V(\delta(\bm{s}, \bm{x}, \bm{y}))$ solves $G_k$ for $k=V(\bm{s_1})$.
\end{proof}

Solutions to the safety games and, by extension, to the \LTL synthesis problem can be constructed from solutions to the Bellman equations. In this step, it is crucial that the transition model ($T$) is known to the agent.
%
%

\subsection{Value Iteration for Safety Games}
%
%
Value iteration for MDPs updates the Q-values of states by performing one-step lookaheads.
We naturally adapt this idea to develop a value iteration algorithm for safety games.
The Q-value estimate for a pair $(\bm{s}, \bm{x})$ is
$\hat{{Q}}(\bm{s},\bm{x}) = \max(\idx(\bm{s}), \min_{\bm{y}}\max_{\bm{x'}} Q({T}(\bm{s}, \bm{x}, \bm{y}),\bm{x'})$.
Intuitively, $\hat{{Q}}(\bm{s},\bm{x})$
performs one-step lookaheads from $\bm{s}$, and 
selects the 
lowest
Q-value 
that the system player could enforce
in an adversarial setting.
Then, the Q-value ${Q}(\bm{s},\bm{x})$ is updated to $\hat{{Q}}(\bm{s},\bm{x})$.
When the Bellman equations have a solution, 
the Q value updates converge in safe states and yield solutions to the associated safety games (Theorem \ref{thm:bellman_one_ucw_convergence}).


\begin{theorem}\label{thm:bellman_one_ucw_convergence}
If the Bellman equations in Theorem \ref{thm:bellman_one_ucw} have a solution, then the Q-value updates below make $V(\bm{s_1})$ converge to a bounded value, 
provided that the Q-value in each state is updated sufficiently often, and that the Q-value function is initialized to non-infinite values.
\[
Q(\bm{s}, \bm{x}) \leftarrow \max(\idx(\bm{s}), \min_{\bm{y}}\max_{\bm{x'}} Q({T}(\bm{s}, \bm{x}, \bm{y}),\bm{x'})
\]
Furthermore, policy
$\pi(\bm{s}, \bm{x}) \defeq \arg\min_{\bm{y}} V({T}(\bm{s}, \bm{x}, \bm{y}))$ 
converges to a solution the safety game $G_k$ for some $k<\infty$.
\end{theorem}
\begin{proof}
To obtain the desired result, we show that the Bellman backup operator $BV(\bm{s}) = \max_x \max(\idx(\bm{s}), \min_{\bm{y}} V({T}(\bm{s}, \bm{x}, \bm{y}))$ is a \emph{contraction} in the set of safe states $S_\mathit{safe} \subseteq S$. 
Let $V^*$ be a solution to the Bellman equations, i.e., $BV^* = V^*$.
By using the property $|\max_x f(x) - \max_x g(x)| \leq \max_x | f(x) - g(x)|$, and observing that $V^*$ is finite in safe states, and infinite in unsafe states, it can be shown that $\max_{\bm{s}\in S_\mathit{safe}} | BV(x) - V^*(x)| \leq \max_{\bm{s}\in S_\mathit{safe}} | V(x) - V^*(x)|$.
Hence, the Q value updates have to converge to finite values in the set of safe states, bounded by some $k < \infty$.
It follows straightforward that the policy
$\pi(\bm{s}, \bm{x}) \defeq \arg\min_{\bm{y}} V({T}(\bm{s}, \bm{x}, \bm{y}))$ 
converges to a solution to the safety game $G_k$.
\end{proof}

\section{DQS: Deep Q-learning for LTL Synthesis}
\label{sec:dqs}

State-of-the-art exact methods for \LTL synthesis, and bounded synthesis in particular, have limited scalability. In part, this is due to the (potentially, doubly-exponential) size of the search space. We can expect, thus, similar challenges with the value iteration algorithm presented in Section \ref{sec:dyncamic_programming_for_ltl_synthesis}.
The objective in this paper is not to do value iteration for safety games. Rather, the Bellman equations (Theorems \ref{thm:bellman_one_ucw}) and the convergence of value iteration (Theorem \ref{thm:bellman_one_ucw_convergence}) set the theoretical foundations that we need to do (deep) Q-learning.

In this section we present a method to \emph{learn} a Q function, with a variant of (deep) Q-learning adapted for safety games.
Why do we want to do Q-learning, if we know the transition model? 
Our aim is not to learn a solution to the Bellman equations. Instead, we want to compute \emph{good enough} approximations that provide good guidance.
With this goal in mind, training a neural network seems to be a reasonable approach because neural networks can capture the structure of a problem and do inference.
%
%
At inference time, the network provides guidance that can be used to construct solutions.


Our approach, which we describe below, constitutes the first attempt to integrate search and learning to do \LTL synthesis.
We provided and elegant and extensible algorithm with wich a variety of search and learning algorithms can be deployed.
For example, the guidance obtained with the trained neural network can be used in combination with AND/OR search techniques such as AO$^*$.
Similarly, the guidance can be used as heuristics with more sophisticated search techniques such as LAO$^*$ \cite{HansenZ01a} in service of \LTL synthesis.
In recent research we exploited the correspondence between \LTL synthesis and AI automated planning (see, e.g., \cite{cam-mci-ijcai19}) to reduce \LTL synthesis specifications into fully observable non-deterministic planning problems that can be solved with off-the-shelf planners  \cite{cam-bai-mui-mci-icaps18,cam-mui-bai-mci-ijcai18}. The heuristic obtained with our trained neural network can be used to guide planners, also in combination with domain-independent heuristics.
Similar techniques than those that we present here can be used to learn heuristics and guide planners in planning problems with \LTL goals (see, e.g., \cite{cam-tri-mui-bai-mci-aaai17,cam-mci-ijcai19}.

\subsection{Deep Learning for LTL Synthesis}

We propose the use of a neural network to approximate the Q function.
%
%
We mainly base our inspiration on recent success in \emph{Deep Q-learning for MDPs} (DQN), where a neural network is used to approximate the Q function in reinforcement learning for MDPs \cite{MnihKSRVBGRFOPB15_dqn}.
%
Our approach shares other commonalities with existing differential approaches to program synthesis that use deep learning.
Like the \emph{Neural Turing Machine}  \cite{GravesWD14}, which augments neural networks with external memory. In comparison, our approach uses automata as compact memory.
%

For its similarities with DQN, we name our \emph{Deep Q-learning for \LTL Synthesis} approach DQS.
Like DQN, DQS uses a neural network to approximate the Q function. The network is trained in batches of \emph{states} obtained from a prioritized experience replay buffer.
Experience is obtained by running search episodes, in which the moves of the environment and system players are simulated according to some exploration policy.
The Q-values estimated by the neural network are used to guide search, as well as to extract solutions.
%
Like DQN, DQS is not guaranteed to converge to a solution to the Bellman equations. However, \emph{good enough} approximations may provide effective guidance, and yield correct solutions to the safety game and synthesis problems.
We provide further details of the components of DQS below.


\paragraph{Q-value network:}
The Q-value network $\bm{Q}_\theta$, 
with parameters $\theta$, takes as input a state vector 
%
$\bm{s}$, and outputs a vector $\bm{Q}_\theta(\bm{s}) = \langle  Q_\theta(\bm{s},1),\ldots, Q_\theta(\bm{s}, D) \rangle$.
Each $Q_\theta(\bm{s},d)$ is an estimation of $Q(\bm{s}, \bm{x})$, if $\bm{x}$ is a binary representation of $d$.
We sometimes abuse notation, and confuse $d$ with its binary representation $\bm{x}$.
The input of $\bm{Q}_\theta$ has the same dimension as the game states, i.e., the number of automaton states in $\C{A}_\varphi$ -- or the sum of automata states, if formula decompositions are exploited.
The output has $D = 2^{|\C{X}|}$ neurons.
The network weights are intitialized to gaussian values close to zero.

\paragraph{Training episodes:}
Training episodes start in the initial state of the game, $\bm{s_1}$. 
Environment and system' actions are simulated according to an exploration policy. New states are generated according to the game transition function, ${T}$, and added into a prioritized experience replay buffer.
Episodes last until a horizon bound is reached, or a state $\bm{s} = \langle -1, \ldots, -1 \rangle$ is reached, or a state $\bm{s}$ with $\idx(\bm{s}) > K$ is reached for some hyperparameter $K$.
After that happens, a new training episode starts.

\paragraph{Exploration policy:}
Different exploration policies can be designed. Here, we use an adaptation of the epsilon-greedy policy commonly used in reinforcement learning. 
At the beginning of each episode, with probability $\mu$, we set an \emph{epsilon-greedy} exploration policy for the environment player. Otherwise, we set it to be \emph{greedy}.
We do the same for the system player, with independent probability.
Greedy environment policies are $\pi(\bm{s}) = \arg\max_{\bm{x}} Q_\theta(\bm{s},\bm{x})$. For the system, greedy policies are $\pi(\bm{s},\bm{x}) = \arg\min_{\bm{y}}  \max_{\bm{x'}} {Q}_{\theta}^{(\bm{x'})}({T}(\bm{s}, \bm{x}, \bm{y}), \bm{x'})$.
Epsilon-greedy policies select an action at random with probability $\epsilon$, and otherwise act greedily.
We use $\mu = \epsilon = 0.2$.

\paragraph{Batch learning step:} 
Learning is done in batches of states, sampled from a prioritized experience replay buffer.
A learning step involves computing, for each state $\bm{s}$ in the batch, a new estimate of the Q-values, $\hat{Q}_\theta(\bm{s}, \bm{x})$.
%
%
%
The updates in each ${\hat{Q}}_{\theta}(\bm{s}, \bm{x})$ do a one-step lookahead as follows:

\begin{align*}
    \bm{Q}_{\theta}(\bm{s}) &\xleftarrow{\alpha} \langle {\hat{Q}}_{\theta}(\bm{s},1), \ldots, {\hat{Q}}_{\theta}(\bm{s}, 2^{|\C{X}|}) \rangle \\
    %
    {\hat{Q}}_{\theta}(\bm{s}, \bm{x}) &= -1 \text{, if } \bm{s} = \la -1, \ldots, -1 \ra \\
    {\hat{Q}}_{\theta}(\bm{s}, \bm{x}) &= \min(K, \max(\idx(\bm{s}), \\
    & \min_{\bm{s'} = {T}(\bm{s}, \bm{x}, \bm{y})} \max_{\bm{x'}} \round( {Q}_{\theta}(\bm{s'}, \bm{x'}) ))) \text{, otherwise. } 
\end{align*}
where $\round(x)$ returns the closest integer to $x$. The values are rounded, as co-B\"uchi indexes are integers.
%
To prevent the Q-value to be learned from diverging to infinite, we bound the values by $K< \infty$.
The idea is to
deem all states $\bm{s}$ with $ V(\bm{s}) \geq K$ as losing states, and focus on producing good Q-value estimates in game states that have co-B\"uchi index lower than $K$.
The batch learning step updates the network weigths as to minimize the sum of TD errors accross all the states in the batch, with some \emph{learning rate} $\alpha\in(0,1)$. 
The TD error in state $\bm{s}$ is $\Sigma_{\bm{x}} |{Q}(\bm{s},\bm{x}) - \hat{{Q}}(\bm{s},\bm{x})|$.
%
To incentive solutions with low co-B\"uchi indexes, we perform L2 regularisation.
We do a batch learning step every four exploration steps, which has reported good results in DQN \cite{MnihKSRVBGRFOPB15_dqn}.




\paragraph{Prioritized experience replay buffer:}
Explored states $\bm{s}$ along episodes are added into a \emph{prioritized experience replay} buffer \cite{SchaulQAS15_prioritized_experience_replay}, with a preference value that equals their current TD error $\Sigma_{\bm{x}} |{Q}(\bm{s},\bm{x}) - \hat{{Q}}(\bm{s},\bm{x})|$.
In a learning step, a batch of states is sampled in a manner that higher preference is given to those states with higher TD error.
After a batch learning step, the preference value
of the sampled states
are
updated to their new TD error.

\paragraph{Double DQS:}
We investigate an enhancement of DQS. We borrow ideas from \emph{Double DQN} (DDQN) for MDPs, where a target network is used to stabilize learning \cite{Hasselt10_ddqn}.
In \emph{Double DQS} (DDQS) we use a target network $\bm{Q}_t$ to estimate the Q-values in the one-step lookaheads, and update $\bm{Q}_t$ to $\bm{Q}_\theta$ 
at the end of each episode.
In DDQS the equations of the batch learning step  become:
%
%
\begin{align*}
            {\hat{Q}}_{\theta}(\bm{s}, \bm{x}) &= \min(K, \max(\idx(\bm{s}), \\
    & \min_{\bm{s'} = {T}(\bm{s}, \bm{x}, \bm{y})} \max_{\bm{x'}} \round( {Q}_t(\bm{s'}, \bm{x'}) ))) 
\end{align*}

\subsection{Solution Extraction and Verification}

An advantage of knowing the transition function $T$ is that solutions extracted from the neural network can be verified for correctness.
At the end of each episode, we verify whether the greedy policy for the system player obtained from the Q-value network ${\bm{Q}}_\theta$ is a solution to the safety game $G_K$.
Recall that a greedy sytem policy is $\pi(\bm{s},\bm{x}) = \arg\min_{\bm{y}}  \max_{\bm{x'}} {Q}_{\theta}^{(\bm{x'})}({T}(\bm{s}, \bm{x}, \bm{y}), \bm{x'})$.
The verification step can be performed by 
doing an exhaustive enumeration of all
reachable game states $\bm{s}$ by $\pi$, from the initial state $\bm{s_1}$, and checking that all have $\idx(\bm{s}) < K$.

\paragraph{Learning from losing gameplays:}
The verification step fails when it encounters a losing partial play, that is, a partial play 
$\play = \bm{s_1} \cdots \bm{s_n}$ with $\idx(\bm{s_n}) = K$.
When this occurs, we run a series of $n$ batch learning steps.
Each batch learning step includes samples from the prioritized experience replay buffer and a state in $\rho$, starting backwards from $\bm{s_n}$.

\subsection{Stronger Supervision Signals} 

One of the challenges in reinforcement learning is having to deal with sparse rewards.
In our learning framework for safety games there are no rewards, but rather supervision signals that enforce the Q-value estimates in a state $\bm{s}$ are not lower than its co-B\"uchi index, $\idx({\bm{s}})$.
Arguably, the problem of sparsity may also manifest in our approach when 
the co-B\"uchi index of visited states along training episodes face sparse \emph{phase transitions} (i.e., infrequent changes). In consequence, it may take a large number of episodes for these values to be propagated in the Q-value network.

We propose the use of a \emph{potential} function $\Phi: S \rightarrow [0,1)$ to provide stronger supervision signal in the learning process. Intuitively, $\Phi(\bm{s})$ indicates how close the co-B\"uchi index of a state $\bm{s} = \la s^{(1)}, \ldots, s^{(m)} \ra$ is from experiencing a phase transition that increments its value.
Formally:
\begin{align*}
    \Phi(\bm{s}) &\defeq 
    \begin{cases}
        0 \text{~~~if } \idx(\bm{s}) = -1 \\
        \max \{ 1/(d(q_i)+1) \vert s^{(i)} = \idx(\bm{s}) \} & \text{otherwise}
    \end{cases}
\end{align*}
where $d(q)$ is the distance (minimum number of outer transitions) between automaton state $q$ and a rejecting state. Without loss of generality, we presume in a \UCW rejecting states are reachable from any state, and thus $d(q)$ is well defined.


The learning updates are redefined, as shown below, to include the supervision given by the potentials. Note, 
this time 
we take the integer part of ${Q}_{\theta}$, $\floor( {Q}_{\theta}(\bm{s'}, \bm{x'}))$.
\begin{align*}
    \bm{Q}_{\theta}(\bm{s}) &\xleftarrow{\alpha} \langle {\hat{Q}}_{\theta}(\bm{s},1), \ldots, {\hat{Q}}_{\theta}(\bm{s}, 2^{|\C{X}|}) \rangle \\
    {\hat{Q}}_{\theta}(\bm{s}, \bm{x}) &= -1 \text{, if } \bm{s} = \la -1, \ldots, -1 \ra \\
    {\hat{Q}}_{\theta}(\bm{s}, \bm{x}) &= \min(K, \max(\idx(\bm{s}), \\
    \min_{\bm{s'} \defeq {T}(\bm{s}, \bm{x}, \bm{y})}& \max_{\bm{x'}} ( \Phi(\bm{s'}) + \floor( {Q}_{\theta}(\bm{s'}, \bm{x'}) ))) \text{, otherwise. } 
\end{align*}

\section{Experiments}

We implemented our Deep Q-learning approach to \LTL synthesis.
%
%
We used 
\emph{Spot} \cite{duret-spot} to transform \LTL formulae into \UCW automata, and
%
Tensorflow 2.0 library for deep learning.
%
Experiments were conducted using 2.4GHz CPU Linux machines
with 10GB of memory.


The purpose of our experiments was not to compete with state-of-the-art tools for \LTL synthesis---these are much faster. Rather, we wanted to evaluate the potential for providing effective search guidance of our neural-based approach.

\paragraph{Hyperparameters:}
%
We experimented with different network sizes. Interestingly, very small networks were often sufficient for learning good guidance.
We fixed a network with two dense hidden layers. The number of neurons corresponded to the input size---i.e., number of \UCW automaton states.
%
%
We employed an $\epsilon$-greedy exploration policy with $\mu=\epsilon=0.2$, and Adam optimizer with learning rate adjusted to $\alpha = 0.005$. 
We set $K=4$ and a horizon bound of $50$ timesteps.
These numbers were set large enough to find solutions.
Search stopped after 1000 episodes.
We used a prioritized experience replay buffer with batch size of 32 states. The learning step in $\bm{Q}_\theta$ was done every 4 timesteps.
In DDQS, the target network $\bm{Q}_t$ was updated at the end of each episode.


\paragraph{Benchmarks:}
We evaluated our system on a family of 19 \emph{lilydemo} benchmark problems retrieved from the \LTL synthesis competition SYNTCOMP \cite{syntcomp}.

\paragraph{Configurations:}
We tested the following configurations:
\begin{itemize}
    \item DQS[$-$]: DQS, reusing losing gameplays for learning.
    \item DDQS: Implementation of DQS with a target network.
    \item DDQS[$-$]: DDQS, reusing losing gameplays for learning.
    \item DDQS[$-$, $\phi$]: Like DDQS[$-$], also using potentials for a stronger supervision signal.
\end{itemize}
In addition, we tested the configurations above using automata decompositions of the \LTL specification.
We refer to those configurations as \emph{dec-DDQS}, etc.
Each configuration was tested in each benchmark a total of 20 times.

\paragraph{Impact of using a target network:}
The Q-values learned using a target network were more realistic than those learned without a target network, which tended to learn higher Q-values.
These observed results resonate with the \emph{optimistic} behaviour often manifested by DQN without the use of a target network.
In terms of performance, DDQS was able to learn faster than DQS (i.e., wth a fewer number of episodes and batch training steps), but the differences were not huge in the benchmarks being tested (see Figure \ref{fig:experiments} (left and right)).

\begin{figure}[t]
    \centering
    \begin{subfigure}[t]{0.5\columnwidth}
        \centering
        \includegraphics[width=\textwidth]{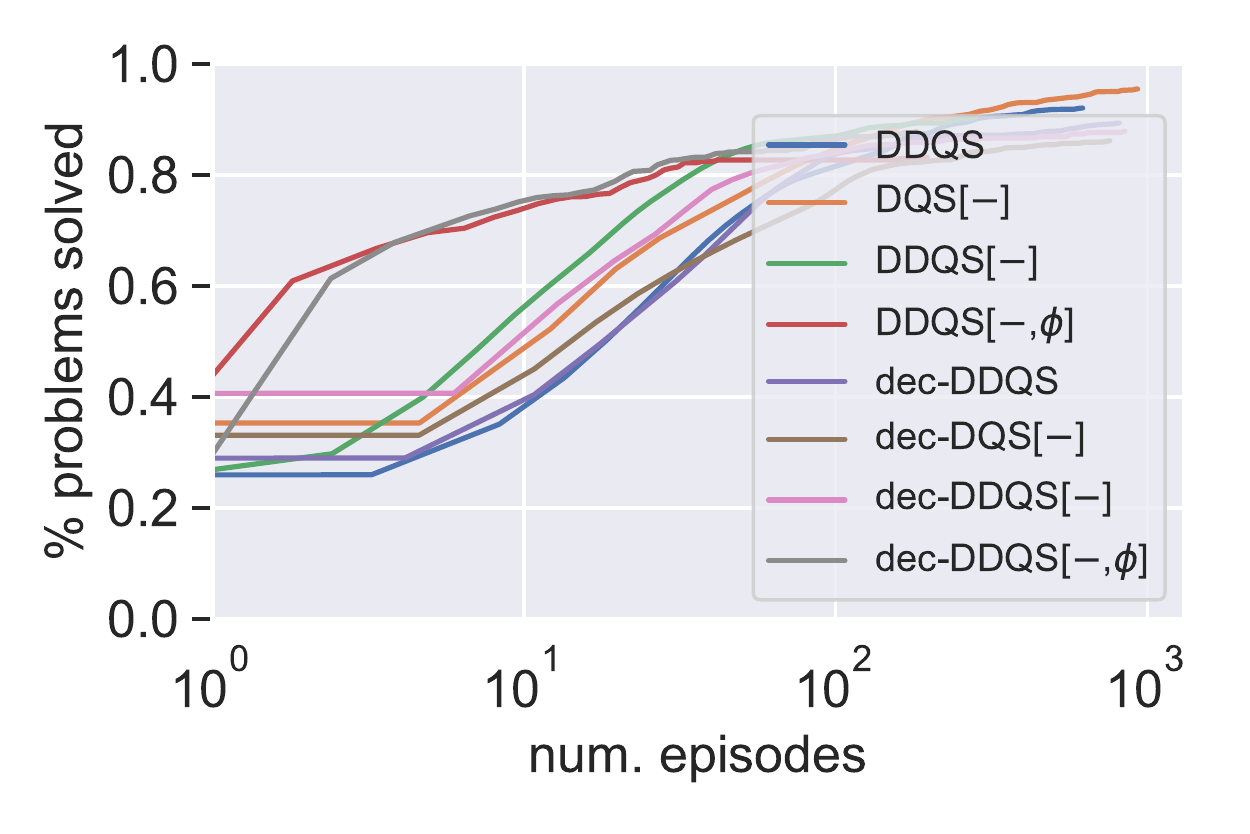}
        \label{fig:n_episodes_cumulative}
    \end{subfigure}%
    \begin{subfigure}[t]{0.5\columnwidth}
        \centering
        \includegraphics[width=\textwidth]{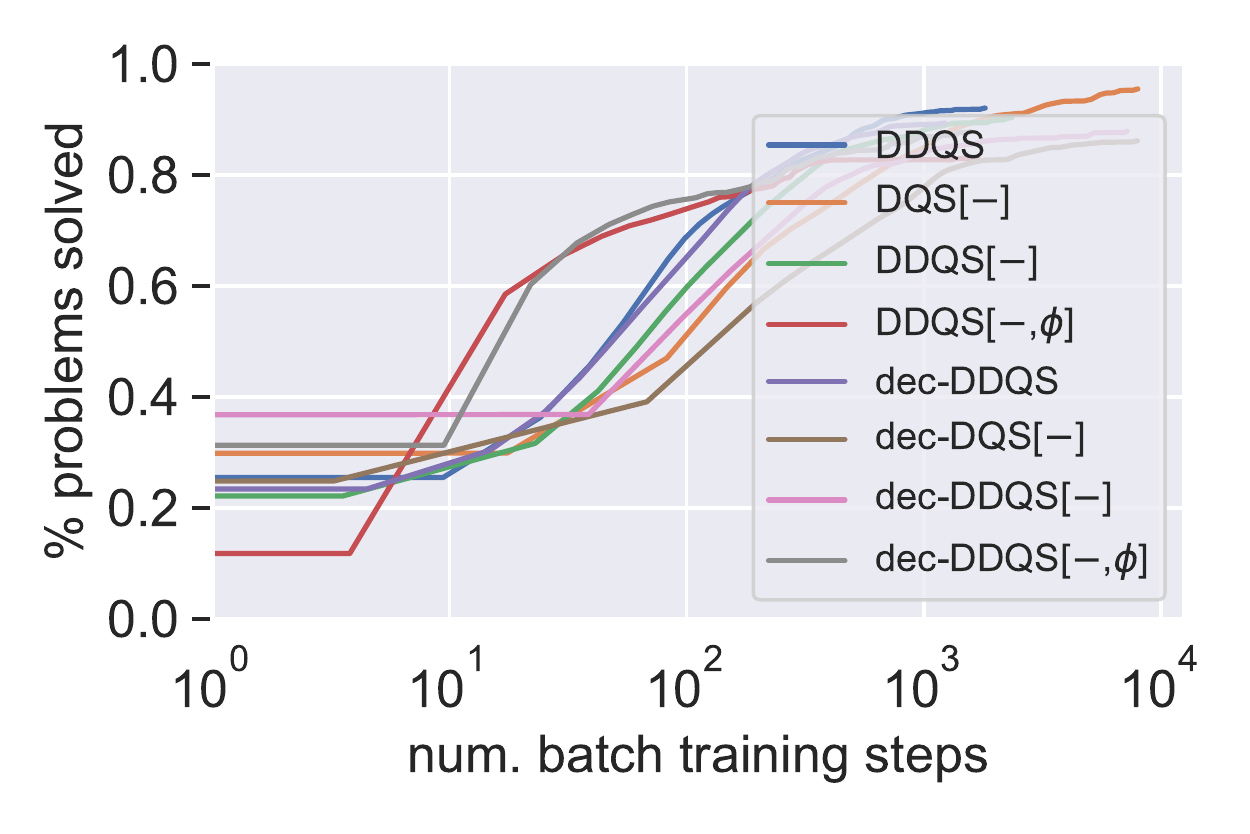}
        \label{fig:n_batch_training_steps}
    \end{subfigure}%
    \caption{
    Problems solved with respect to the number of episodes (left) and the number of batch learning steps (right).
    We compared different configurations of DQS, using a target network (DDQS), reusing losing gameplays for learning ($-$), using potentials for a stronger supervision signal ($\phi$), and using automata decompositions (dec). Each configuration was tested a total of 20 times in each of the 10 \emph{lilydemo} benchmark problems.
    }
    \label{fig:experiments}
\end{figure}

\paragraph{Impact of using decompositions:}
The use of automata decompositions may be necessary to scale to those problems with large specifications that, for computational limitations, cannot be transformed into a single automaton.
Our learned greedy policies were not able to find exact solutions in large specifications where automata decompositions become necessary. Thus, further work needs be done to better exploit the guidance of the neural network in the search for solutions.
Still, we conducted a study of the the impact of automata decompositions in  the \emph{lilydemo} benchmark set.
The use of automata decompositions translated into a larger number of automaton states, and therefore, required a larger number of input neurons (see Figure \ref{fig:experiments_2} (left)). The use of automata decompositions also required more training episodes and steps, but not a huge number compared to single-automaton transformations of the \LTL formula (see Figures \ref{fig:experiments} and \ref{fig:experiments_2}).

\paragraph{Impact of using losing gameplays:}
By exploiting losing gameplays, our approaches were able to learn with fewer learning episodes and batch learning steps (see Figure \ref{fig:experiments}). 

\begin{figure}[t]
    \centering
    \begin{subfigure}[t]{0.5\columnwidth}
        \centering
        \includegraphics[width=\textwidth]{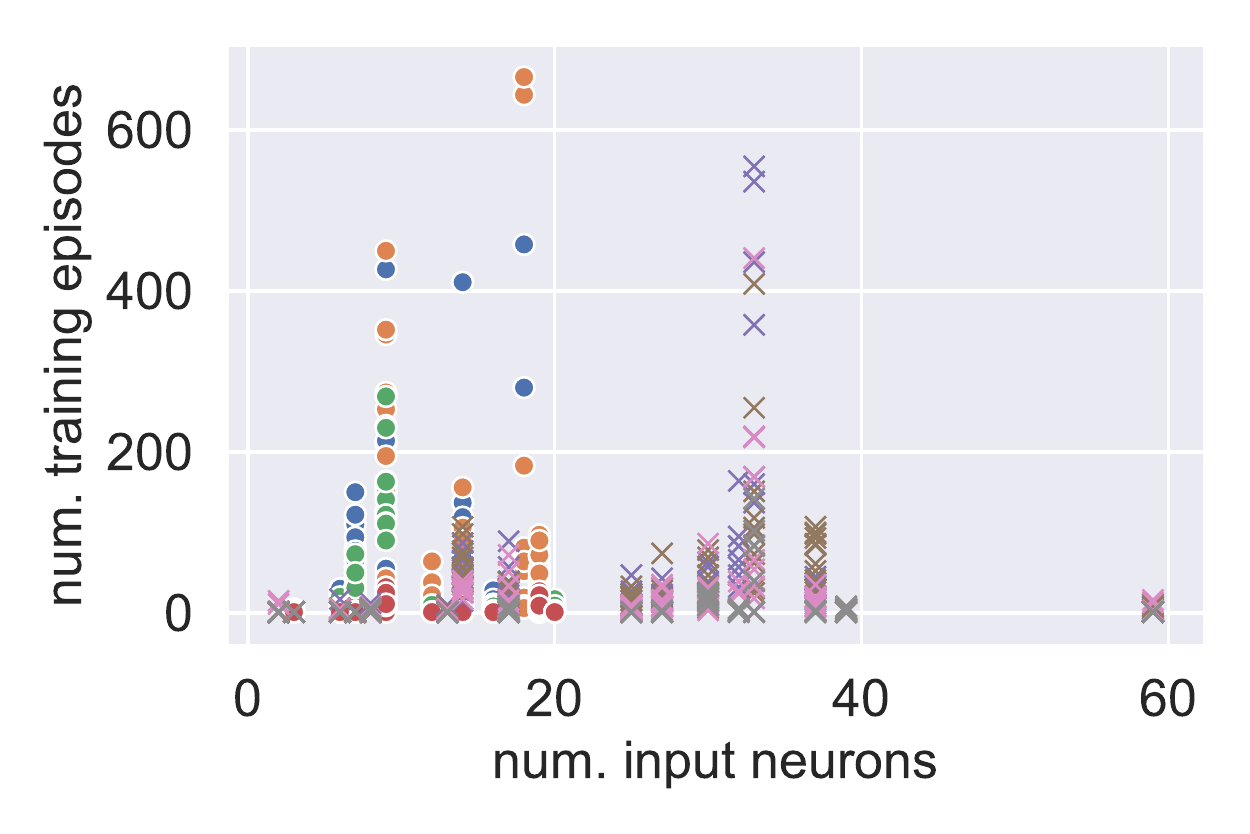}
        \label{fig:n_episodes_vs_size}
    \end{subfigure}%
    \begin{subfigure}[t]{0.5\columnwidth}
        \centering
        \includegraphics[width=\textwidth]{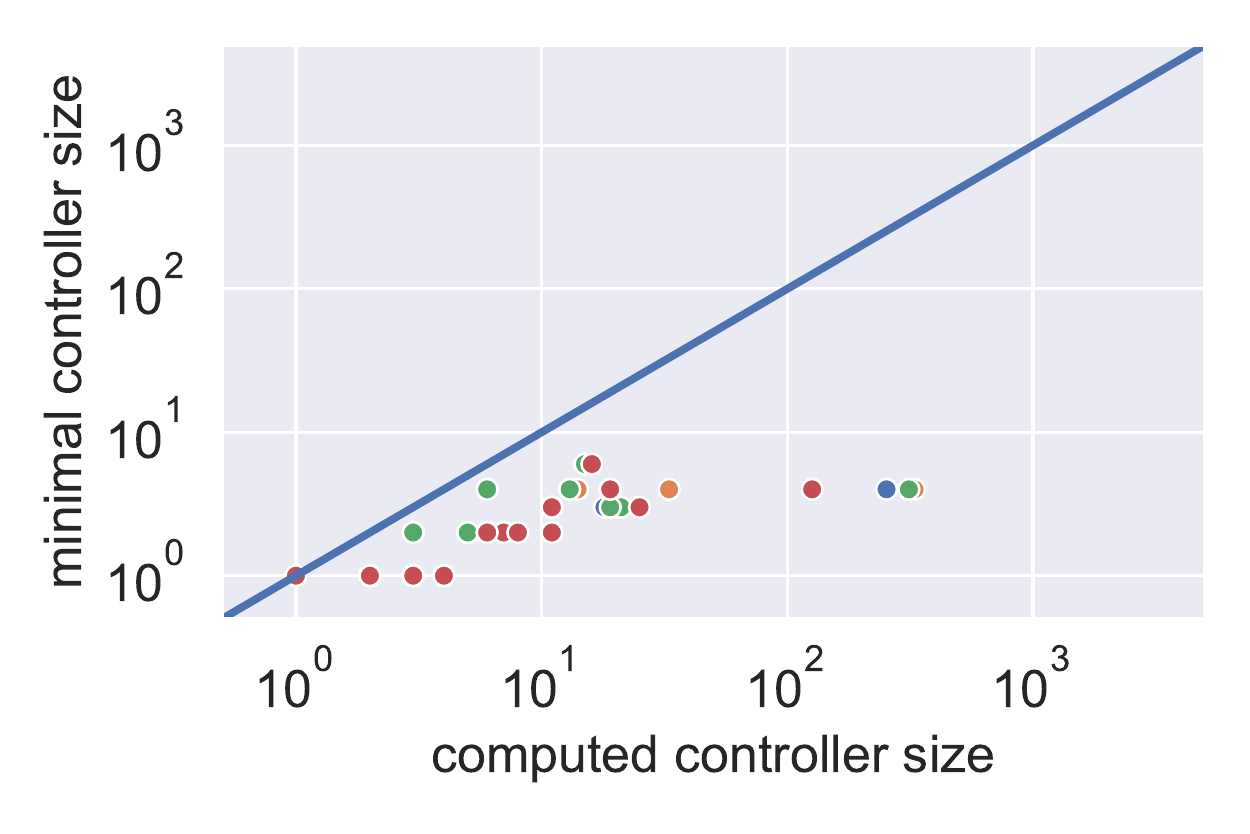}
        \label{fig:controller_sizes}
    \end{subfigure}
    \caption{
    Left: training episodes with respect to the input size of the neural network. Datapoints from algorithms that do not use decomposition are represented with circles, and otherwise represented with crosses. 
    Right: size of the controllers obtained with our neural-based methods, compared with the minimal size of the controllers as reported in SYNTCOMP \cite{syntcomp}.
    }
    \label{fig:experiments_2}
\end{figure}

\paragraph{Impact of using potentials:}
The use of a stronger supervision signal in the form of potentials greatly improved the learning process.
We can observe in Figure \ref{fig:experiments} (left and right) that DDQS[$-$,$\phi$] and dec-DDQS[$-$,$\phi$] greatly outperformed all other configurations that do not make use of potentials.

\paragraph{Size of controllers:}
We compared the size of the controllers obtained with the learned greedy policies in the different configurations of our learning-based system, with the minimal size of the controllers that solve the benchmark problems being tested (as reported in \cite{syntcomp}).
The controllers that we produced are significantly larger than the size of the minimal controllers for each benchmark problem (see Figure \ref{fig:experiments_2} (right)).
This suggests that the performance of our system may benefit from the combination of more sophisticated search techniques---other than just epsilon-greedy exploration policies, and greedy execution policies---to prune the search space.

\section{Discussion and Future Work}


We
addressed 
\LTL synthesis, a formulation of program synthesis from specification that is 2EXP-complete, and
that automatically generates programs that are correct by construction.
%
Exact methods have limited scalability.
The development of novel techniques with the potential to scale is crucial.

We presented the first approach to \LTL synthesis that combines two scalable methods: search and learning.
Our novel approach reformulates \LTL synthesis as an optimization problem. 
We reformulated \LTL synthesis as a dynamic programming problem,
and explored deep Q-learning to approximate solutions. Ultimately, our objective was to train neural networks to provide good guidance. 
%
%
%
Our approach shares commonalities with neuro-symbolic and neural-guided search approaches in using learned properties to guide search. 
Like the \emph{Neural Turing Machine}, which augments neural networks with external memory \cite{GravesWD14}, 
we use \emph{automata} as compact memory.


We were interested in evaluating the potential for providing effective search guidance of our neural-based approach.
%
In our experiments, 
we 
solved synthesis benchmarks 
by virtue of simply executing policies that acted greedily 
with respect to the neural network guidance.
In many cases, the network was trained using only a few dozen episodes in problems whose solutions are simple, but where the size of the search space is $\bigO(2^{30})$ or more.
%
%
%
%
%
%
Furthermore, we found that simple enhancements, like reusing losing plays, could significantly improve performance. We also found that using potentials to strengthen the supervision signal proved extremely beneficial.
While we did not solve the largest benchmark problems in SYNTCOMP, and our approach did not manifest state-of-the-art performance, we believe that the combination of learning and search proposed here provides a foundation for \LTL synthesis and opens a new avenue for research.

\bibliography{references}
\bibliographystyle{apalike}

\end{document}